\newif\iffulledition 
\fulleditiontrue 

\documentclass[\iffulledition nonacm,\fi sigconf]{acmart}

\newtheorem{theorem}{Theorem}
\newtheorem{corollary}{Corollary}

\usepackage{graphicx}
\graphicspath{{figures/}}
\usepackage{booktabs}
\usepackage{subcaption}
\usepackage{array}
\usepackage{multirow}
\usepackage{tabularx}
\newcolumntype{Y}{>{\centering\arraybackslash}X}
\newcolumntype{M}[1]{>{\centering\arraybackslash}m{#1}}

\iffulledition
\usepackage{algorithmic}
\usepackage[linesnumbered,ruled,vlined]{algorithm2e}
\SetKwInput{KwInput}{Input}  
\SetKwInput{KwOutput}{Output}
\fi
\usepackage{pgfplots}
\pgfplotsset{compat=newest}

\setcopyright{rightsretained}
\copyrightyear{2023}
\acmYear{2023}
\acmDOI{10.1145/3555776.3578601}

\acmConference[SAC '23]{The 38th ACM/SIGAPP Symposium On Applied Computing}{March 27-April 2, 2023}{Tallinn, Estonia}
\acmBooktitle{The 38th ACM/SIGAPP Symposium on Applied Computing (SAC '23), March 27-April 2, 2023, Tallinn, Estonia}
\acmISBN{978-1-4503-9517-5/23/03}

\begin{document}

\title{Proof of Swarm Based Ensemble Learning for Federated Learning Applications}
\iffulledition
\titlenote{This is the full edition of a 4-page poster paper published at the Proceedings of the 38th ACM/SIGAPP Symposium on Applied Computing (SAC '23), which can be accessed via the following DOI link: \url{https://doi.org/10.1145/3555776.3578601}.}
\titlenote{Please cite this paper as follows: Ali Raza, Kim Phuc Tran, Ludovic Koehl and Shujun Li, ``Proof of Swarm Based Ensemble Learning for Federated Learning Applications,'' arXiv:2212.xxxxx [cs.LG], 2022, \url{https://doi.org/10.48550/arXiv.2212.xxxxxx}, a shorter edition was published in Proceedings of the 38th ACM/SIGAPP Symposium on Applied Computing (SAC '23), ACM, 2023, \url{https://doi.org/10.1145/3555776.3578601}.}
\fi

\author{Ali Raza}
\authornote{Corresponding co-authors: Ali Raza and Shujun Li.}
\email{ar718@kent.ac.uk}
\orcid{0000-0001-8326-8325}
\affiliation{
  \institution{Univ.~Lille, ENSAIT, ULR 2461 -- GEMTEX -- Génie et Matériaux Textiles, France}
}
\affiliation{%
  \institution{Institute of Cyber Security for Society (iCSS) \& School of Computing, University of Kent, UK}
}

\author{Kim Phuc Tran}
\email{Kim-phuc.tran@ensait.fr}
\orcid{0000-0002-6005-1497}
\author{Ludovic Koehl}
\email{ludovic.koehl@ensait.fr}
\orcid{0000-0002-3404-8462}
\affiliation{
  \institution{Univ.~Lille, ENSAIT, ULR 2461 -- GEMTEX -- Génie et Matériaux Textiles, France}
}

\author{Shujun Li}
\iffulledition 
\authornotemark[3]
\fi
\email{S.J.Li@kent.ac.uk}
\orcid{0000-0001-5628-7328}
\affiliation{%
  \institution{Institute of Cyber Security for Society (iCSS) \& School of Computing, University of Kent, UK}
}

\begin{abstract}
Ensemble learning combines results from multiple machine learning models in order to provide a better and optimised predictive model with reduced bias, variance and improved predictions. However, in federated learning it is not feasible to apply centralised ensemble learning directly due to privacy concerns. Hence, a mechanism is required to combine results of local models to produce a global model. Most distributed consensus algorithms, such as Byzantine fault tolerance (BFT), do not normally perform well in such applications. This is because, in such methods predictions of some of the peers are disregarded, so a majority of peers can win without even considering other peers' decisions. Additionally, the confidence score of the result of each peer is not normally taken into account, although it is an important feature to consider for ensemble learning. Moreover, the problem of a tie event is often left un-addressed by methods such as BFT. To fill these research gaps, we propose PoSw (Proof of Swarm), a novel distributed consensus algorithm for ensemble learning in a federated setting, which was inspired by particle swarm based algorithms for solving optimisation problems. The proposed algorithm is theoretically proved to always converge in a relatively small number of steps and has mechanisms to resolve tie events while trying to achieve sub-optimum solutions. We experimentally validated the performance of the proposed algorithm using ECG classification as an example application in healthcare, showing that the ensemble learning model outperformed all local models and even the FL-based global model. \iffulledition To the best of our knowledge, the proposed algorithm is the first attempt to make consensus over the output results of distributed models trained using federated learning.\fi
\end{abstract}



\maketitle

\section{Introduction}

Machine learning (ML) can improve digital healthcare by providing efficient and accurate solution to different problems~\cite{lee2021application\iffulledition, lysaght2019ai\fi}. Federate learning (FL) has been applied in many healthcare application to solve the issues of centralised machine learning, where a joint machine learning model is trained by distributed peers. This models is then downloaded and personalised by local devices~\cite{xu2021federated\iffulledition , yang2019federated\fi}. FL helps enhance privacy of data owners. Furthermore using FL, distributed data can be used to train robust ML models \iffulledition. Different organisations train robust models for different healthcare applications, for example, Electrocardiogram (ECG) classification, cancer tumour detection etc. Such models are then available as ML-as-a-service. This allow clients to access ML models only via a prediction query interface, which provides predictions (e.g., classification) \fi. However, in such applications, results from a single model cannot be trusted completely because of potential negative consequences of false positives and false negatives. One solution for this problem is to query different models for cross-validation before any results are accepted. However, different models can provide different prediction results and confidence scores for a given input sample. Choosing the right results (prediction) among many can be difficult and challenging. In other words, FL enables collaborative training of joint model but cannot perform \emph{consensus} over the distributed predictions once the global training has been completed and deployed at local devices. The local devices generally personalise the distributively trained model using their local data. 

Methods like Byzantine fault tolerance can address such issues in distributed computing, nevertheless, such methods work based on simple majority voting~\cite{castro2002}, without considering confidence score for results. Confidence scores of results play an important roles, which can be explained by an illustrative example: $n$ peers all with high confidence on a result are clearly better than $n$ peers all with less confidence on the same result. Hence, it is useful for consensus algorithms to consider confidence scores of all participating peers in order to produce more confident consensus results.
 
To the best of our knowledge, there is only limited related work, which tries to achieve consensus during the training phase of multiple machine learning models~\cite{savazzi2020federated\iffulledition, li2020blockchain\fi}. Such methods cannot be applied to scenarios where multiple pre-trained machine learning models work together to achieve a consensus for unseen data, which remains an open research question.

Swarm intelligence, a natural phenomenon in many organisms, has been used to get the (sub-)optimal choice among groups, schools and colonies. \iffulledition For instance, it has been used in robotics to choose the optimal path~\cite{metcalf2019keeping, sonti2021artificial}. \fi Artificial swarm intelligence of distributed models can often achieve superior results over individual models who participate~\cite{chakraborty2017swarm}.
 
In this paper, inspired by swarm intelligence, we propose a novel consensus algorithm called Proof of Swarm (PoSw). The proposed algorithm can be used to obtain (sub-)optimal consensus among all the peers by effectively considering output probability distributions (confidence scores) over all the candidate outputs to obtain an agreement (consensus) among all the peers over the output results. The proposed algorithm does not involves complex computation, so it can be used in resource constraint edge device(s).

\iffulledition
The main contributions of our work are summarised as follow.
\begin{enumerate}
\item We propose a novel lightweight consensus algorithm to achieve a(n) (sub-)optimum consensus among the peer classifiers in a federated learning.

\item We rigorously prove that the proposed algorithm can always converge in a limited number of steps.

\item The proposed algorithm is computationally efficient and hence can be used in resource-constraint devices.

\item We provide experimental results to validate the performance of the proposed algorithm using ECG classification as an example application in healthcare.

\item Due to the distributed nature of FL, the proposed algorithm provides enhanced security against some attacks, e.g., poisoning attacks.
\end{enumerate}
\fi

The rest of the paper is organised as follow. \iffulledition Section~\ref{sec:background} presents background and related work. \fi Section~\ref{sec:PoSw} presents the proposed PoSw consensus method. \iffulledition Section~\ref{sec:cases} presents some case studies using the proposed PoSw consensus method. \fi Section~\ref{sec:experiments} shows experimental results. \iffulledition Some further discussions on the proposed PoSw method are given in Section~\ref{sec:discussions}, including some additional data security and privacy challenges and future research opportunities. \fi The last section concludes the paper.

\iffulledition
\section{Background and Related Work}
\label{sec:background}

In this section we present background and related work.

\subsection{Federated Learning}

Federated learning (FL)~\cite{mcmahan2016federated} collaboratively trains a joint model to achieve robustness, and privacy. In FL the edge devices train a local model using their local data and share the trained parameters with a central server which aggregates the share parameters according to a given aggregation algorithm~\cite{mcmahan2016federated, blanchard2017} to created parameters of a global model. The parameter of the global model are then downloaded to be utilised locally by each edge device. This process repeats with emerging data until a desired level of performance is achieved. FL enhances the privacy of data owners because each edge device do not share its raw data directly with other edge devices in the network.

\subsection{Byzantine Fault Tolerance}
\label{subsec:BFT}

Byzantine fault tolerance (BFT)~\cite{castro2002} is one of the most popular consensus methods used in distributed systems. To achieve consensus using BFT, in a distributed system at least a majority of $\frac{2}{3}$ of all peers should agree on a given decision. However, it cannot achieve a consensus if the majority voting is not achieved or in case of a tie, i.e., 50\% agree and 50\% do not agree on a given decision.

\subsection{Swarm Intelligence}
\label{subsec:si_methods}

Swam intelligence (SI) is being used to solve many optimisation problems. SI works using collective intelligence of groups of agents, such as group of artificial intelligence based decentralised systems~\cite{bonabeau1999}. Agents of a group in SI interact with each other by sharing information among each other in regards to a particular task, followed by execution of various simple tasks by each individual agent. This allows the group of agents (swarm) to solve complex problems with mutual consensus~\cite{bonabeau1999}. Due to its promising results, SI has been used in many applications such as medical dataset classification, moving objects tracking communications, and predictions~\cite{zhang2013}.

A number of researchers have proposed swarm optimisation based collective decision making models~\cite{valentini2017best, valentini2015, hamann2013, grishchenko2021}. For example, Hamann et al.~\cite{hamann2013} proposed an abstract model for collective decision making inspired by urn models. To break a tie, they suggested relying on noise because a real swarm will be noisy. Similarly, Grishchenko et al.~\cite{grishchenko2021} described how gnomes develop their own non-Byzantine leaderless consensus algorithms based on simple rules (e.g., one genome proposes a plans, which then spreads in the whole network using gossips). They also explore Byzantine-ready version of the algorithm where ties are addressed using the rank of the genome that proposed the plan. Nevertheless, such methods are not directly applicable in our application area, which lacks features such as a noisy swarm and rankings of clients/edges. Hence, significant modifications are need to be made in order to adopt such algorithms.

\subsection{Blockchain-based Consensus}
\label{subsec:bc_methods}
Consensus in blockchain involves the agreement of peers in the network about current state of data in the network. Though numerous consensus methods are being used to achieve consensus among the peers in blockchain, the most widely used are Proof of Work (PoW) and Proof of Stake (PoS)~\cite{Bach2018}. PoW works by searching for a value that, when hashed gives a hash with a predefined number of zeros in the prefix of hash (usually accomplished by adding a nonce value). PoS uses stack (wealth, reputation etc.) of peers for validation. Peers with higher stack have higher changes to get selected to validate updates. Blockchain-based consensus have been proposed to address security issues in FL. For example, Mengfan and Xinghua~\cite{federated2022fedbc} proposed FedBC, a gradient similarity based secure consensus algorithm to address byzantine attacks in FL. Nevertheless, blockchain-based consensus algorithms are usually computationally expensive and are not easily scalable. Similar to the other SI methods, due to lack of features such as stack, and high computational power blockchain-based consensus algorithm have limited applicability in our application area, and requires significant amount of modification before using them in such applications.      
\fi

\section{Proposed Method}
\label{sec:PoSw}

We will use an indicative example to explain the proposed PoSw consensus method. Let us suppose that five edge devices have trained a global model for classification of ECG signals into five classes, S, V, F, N, and Q, using federated learning. After receiving the globally trained model, each edge device personalises the global model using its local data. Now, an input sample is given to each of the edge device. Each edge device will output a classification result (confidence score over candidate classes). Here, the output is actually the probability distribution given by the softmax function of the trained model for the given input at each edge device $E_i$ ($i=1, 2, 3, 4, 5$). Suppose that edge devices 1 and 2 predicted class N, while edge devices 3, 4 and 5 predicted class V, Q and F, respectively. This is because the model considers the class with the highest probability as the predicted class. Since not all of the peers have the same output class for the same input, trusting any particular result is not feasible. Therefore, to achieve consensus in this situation, our proposed PoSw method considers the confidence scores of all edge devices' results so that results with higher confidence can be prioritised. Assuming there are $n$ edge devices, the general workflow of the proposed PoSw method can be described as follow. \iffulledition In Section~\ref{sec:cases}, we will discuss some case studies to illustrate more how the proposed PoSw method works.\fi

\begin{enumerate}
\item Each edge device $E_i$ broadcasts $(C_i,p_i)$ to the whole network, where $C_i$ is the local ``best'' class label with the maximum probability $p_i$. If more than one class label has the same maximum probability, randomly choose one.

\item For each unique class label $c\in\{C_i\}_{i=1}^n$, count the number of votes it receives among all edge devices and denoted it by $\#(c)$. Denote the maximum number of votes by $M=\max\{\#(C_i)\}_{i=1}^n$. Now each edge device calculates a set of global ``best'' class labels $\mathcal{C}$ as follows:
    \begin{enumerate}
    \item If there is a single class label $c$ with the maximum number of votes $M$, $\mathcal{C}=\{c\}$.
    
    \item If there are more than one class label with the maximum number of votes $M$, then calculate the sum of the probabilities of each such class label $c$ according to Eq.~\eqref{eq:Psum}. If a single class label $c$ has the maximum sum, then $\mathcal{C}=\{c\}$.

    \begin{equation}\label{eq:Psum}
    P(c) = \sum\nolimits_{\forall i, C_i=c} p_i.
    \end{equation}
    
    \item If more than one class label with the maximum probability sum, then set $\mathcal{C}$ to be the set of all such labels.
    \end{enumerate}

\item Each edge device satisfying $C_i \not\in \mathcal{C}$ performs the \emph{move} function, by assigning $C_i$ to be the next class label $C_i'$ with the next highest probability $p_i'$, and then re-broadcasting $(C_i',p_i')$. If an edge device exhausts all class labels, it goes back to the class label with the highest probability (i.e., ``resets'' the whole process).

\item Repeat the above two steps until the status of the whole network converges, e.g., $\forall i$, $C_i\in\mathcal{C}$.
\end{enumerate}

For the proposed PoSw algorithm, we can prove the following important theorem.

\begin{theorem}
\label{theorem:convergence}
Assuming there are $N>1$ edge devices and $K>1$ class labels, the above-described PoSw algorithm will converge to reach a consensus after at most $K(K-1)$ rounds.
\end{theorem}

\begin{proof}
For the $i$-th round of the algorithm, denote the set of the global best labels by $\mathcal{C}_i$, and assume that $n_i$ edge devices that vote for one of the labels in $\mathcal{C}_i$. If $n_i=N$, the algorithm reaches the end so can stop. Therefore, we now only consider the case of $n_i<N$. In the following, we show for all possible cases, after a finite number of steps, $n_i$ will increase by at least one, i.e., $n_{i+j}\geq n_i+1$, where $j$ is a finite number.

According to the proposed PoSw algorithm, only the $N-n_i$ edge devices that did not vote for any class labels in $\mathcal{C}_i$ should perform the \emph{move} function. Assume after the moves, the new class labels of $N-n_i$ edge devices choose are $C_1, \ldots, C_{N-n_i}$. Consider two different scenarios.

Scenario 1) $\exists c\in\mathcal{C}_i$, which appears at least once in $C_1, \ldots, C_{N-n_i}$: In this case, $n_{i+1}\geq n_i+1$ will always hold since no matter which class label(s) ($\mathcal{C}_i$ or one or more in $C_1', \ldots, C_{N-n_i}'$) is/are selected, the number of votes will be no less than $n_i+1$, the minimum number of votes $c$ gets in the new round.

Scenario 2) $\forall c\in\mathcal{C}_i$, $c$ does not appear in $C_1, \ldots, C_{N-n_i}$: In this case, the number of votes of each global best class label in $\mathcal{C}_i$ remains unchanged. Now let us consider two sub-scenarios.

Scenario 2a) If one or more of $C_1, \ldots, C_{N-n_i}$ get more votes than $n_i$, then $\mathcal{C}_{i+1}$ will change to the set of those new class label(s), and $n_{i+1}\geq n_i+1$ after just one round.

Scenario 2b) If none of $C_1, \ldots, C_{N-n_i}$ get more votes than $n_i$, let us consider all future rounds of the algorithm. If for any round $j>i$, Scenario 1 or 2a happens then $n_j\geq n_i+1$ will hold, therefore, the only possibility for a consensus to not take place will be when the algorithm is ``trapped'' within Scenario 2b forever. Now let us assume that the algorithm is indeed trapped in Scenario 2b forever. In this case, the global best class labels appear in all future rounds as follows:
\[
\overbrace{\mathcal{C}_1, \ldots, \mathcal{C}_1}^{i_1+1\text{ to }i_1+f_1\text{ rounds}},
\ldots,
\overbrace{\mathcal{C}_K, \ldots, \mathcal{C}_K}^{i_{K-1}+1\text{ to }i_{K-1}+f_K\text{ rounds}},
\ldots
\]
Assume $\exists k\in\{1,\ldots,K\}$ so that $f_k\geq K-1$. Then, all the edge devices that did not vote for any class label in $\mathcal{C}_k$ in the $i_k+1$-th round would have exhausted all the remaining $K-1$ candidate class labels, which must include at least one label $c\in\mathcal{C}_k$. If so, the number of votes $c$ gets should have increased by at least one before reaching the $i_k+f_k$-th round. This means that since the $i_1$-th round the number of votes of the global best must have increase at least by one in at most $K(K-1)$ rounds. On the other hands, if the $\forall k\in\{1,\ldots,K\}$ so that $f_k<K-1$, let us prove that $\forall i>j$, $\mathcal{C}_i\cap \mathcal{C}_j=\varnothing$. The nature of being trapped in Scenario 2b is that the global best class label(s) can only change if $P(\mathcal{C}_j)>P(\mathcal{C}_i)$ since the number of votes remains $n_i$. This means that $P(\mathcal{C}_K)>\cdots>P(\mathcal{C}_1)$. Since the probability of any class label is static, the inequality implies $\forall i,j \in \{1, \ldots, K\}$, $\mathcal{C}_i\neq\mathcal{C}_j$. Given there are only $K$ class labels, we have $\cup_{i=1}^K\mathcal{C}_i = \{1, \ldots, K\}$. Now, after $\mathcal{C}_K$, the output of the algorithm will not change since none of the class labels in $\{1, \ldots, K\} - \mathcal{C}_K$ will have a higher probability than $P(\mathcal{C}_K)$. Therefore, $\mathcal{C}_K$ will be the output forever, i.e., $f_K=\infty$, which contradicts to the previous assumption that $f_K<K-1$. Therefore, the algorithm cannot be trapped forever in Scenario 2b and will go to other scenarios in at most $K(K-1)$ rounds, at which point the number of votes will increase by at least one.

Combining all the above scenarios together, we can see the maximum rounds needed to let the number of votes the global best class label(s) to increase by at least one is $\max(1,K(K-1))$. Since $K>1$, we can get $K(K-1)\geq 2$ so the maximum number of rounds is $K(K-1)$.
\end{proof}

A corollary from Theorem~\ref{theorem:convergence} is the following.

\begin{corollary}
\label{corollary:early_stop}
Once the PoSw algorithm produces an output $\mathcal{C}$ with $\lfloor K/2\rfloor+1$ (a simple majority of all votes), $\mathcal{C}$ will be the final converged solution so the algorithm can stop.
\end{corollary}

\iffulledition
\section{Case studies}
\label{sec:cases}

In this section, we provide two case studies to illustrate how the proposed PoSw method works.

\subsection{Case Study 1: When there is no tie}

Figure~\ref{fig:case_study_1_PD} shows the probability distribution of a sample input. According to the proposed PoSw method, as shown in Figure~\ref{fig:case_study_1_algo}, each edge will broadcast its predicted class label with the maximum probability, and will set it as its $C_i$ in Round~1. From Figure~\ref{fig:case_study_1_algo}, it can be seen that the local best class labels for Edges~1 and 2 are both N, while for Edges~3, 4, and 5 they V, Q and F, respectively. Since the class label with the maximum number of votes is N (it has two votes and others have just one), $\mathcal{C}$ is set to be N. Now in Round~2, Edges~3, 4,and 5 will perform the \emph{move} function and update their $C_i$ because their $C_i \not\in \mathcal{C}$. Hence in Round~2, Edges~4, and 5 will update their $C_i$ to N because it is the class label with the second highest probability. For Edge~3, it will update $C_3$ to Q, as it has the second highest probability. After updating $C_i$, each edge will broadcast the new $C_i$ to all peers to determine the new $\mathcal{C}$, which is still N (now with four votes). After Round~2, we already have $\mathcal{C}$ with a simple majority of all votes, so according to Corollary~\ref{corollary:early_stop} we can change all local bests to N and stop.

\begin{figure}[!ht]
\centering
\begin{subfigure}{\linewidth}
  \centering
  \includegraphics[width=\linewidth]{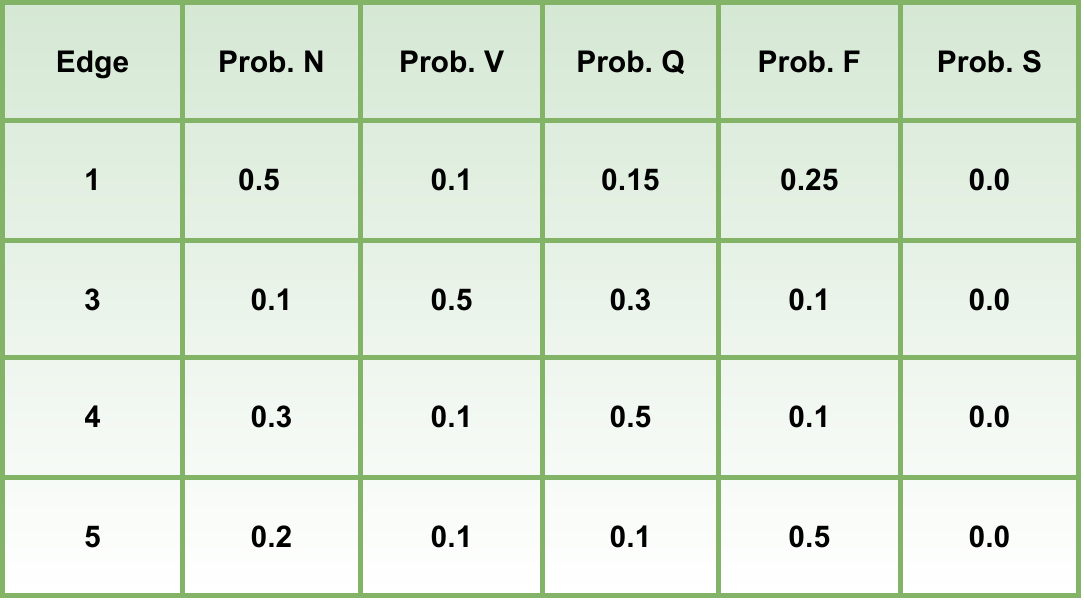}  
  \caption{Input probability distribution}
  \label{fig:case_study_1_PD}
\end{subfigure}
\begin{subfigure}{\linewidth}
  \centering
  \includegraphics[width=\linewidth]{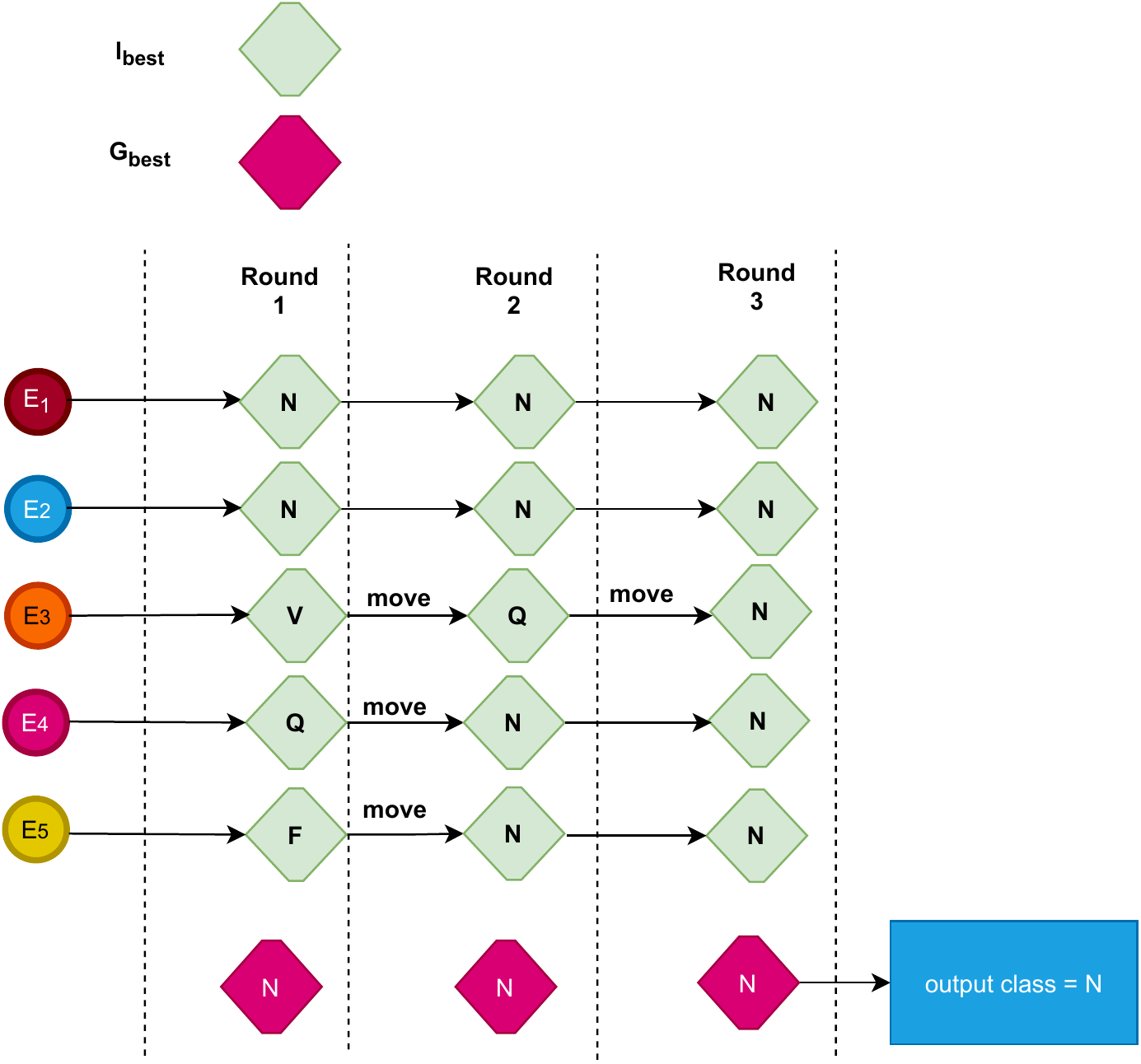}  
  \caption{Proof of Swarm}
  \label{fig:case_study_1_algo}
\end{subfigure}
\caption{Case Study 1.}
\end{figure}

\subsection{Case Study 2: When there is a tie}

Figure~\ref{fig:case_study_2_PD} shows the probability distribution of another sample input. According to the proposed PoSw method, as shown in Figure~\ref{fig:case_study_2_algo}, each edge will broadcast its class label with the maximum probability, and will set it as its $C_i$ in Round~1. From Figure~\ref{fig:case_study_2_algo}, it can be seen that the local best class labels for Edges~1 and 2 are both N, while for Edges~3, 4, and 5 they are V, Q and S, respectively. Since the class label with the maximum number of count is class N (it has a count of 2), N will be set as $\mathcal{C}$. Now in Round~2, Edges~3, 4,and 5 will perform the \emph{move} function and update their $C_i$ because their $C_i \not\in \mathcal{C}$, as shown in Figure~\ref{fig:case_study_2_algo}. Now, in Round~2, there are two candidates for $\mathcal{C}$, N and F, both with two votes. In this case, each edge device will compute $P(\text{N})$ and $P(\text{F})$ using Eq.~\eqref{eq:Psum}. As $P(\text{F})>P(\text{N})$, F will be set as $\mathcal{C}$. Now in Round~3, Edges~1, 2 and 5 will perform the \emph{move} function and update their $C_i$, i.e., to F, Q and F, respectively. Therefore, $G_{\text{best}}$ will now be set to F (with four votes). After Round~3, we already have $\mathcal{C}$ with a simple majority of all votes, so according to Corollary~\ref{corollary:early_stop} we can change all local bests to F and stop.

\begin{figure}[!ht]
\centering
\begin{subfigure}{\linewidth}
  \centering
  \includegraphics[width=\linewidth, height=5cm]{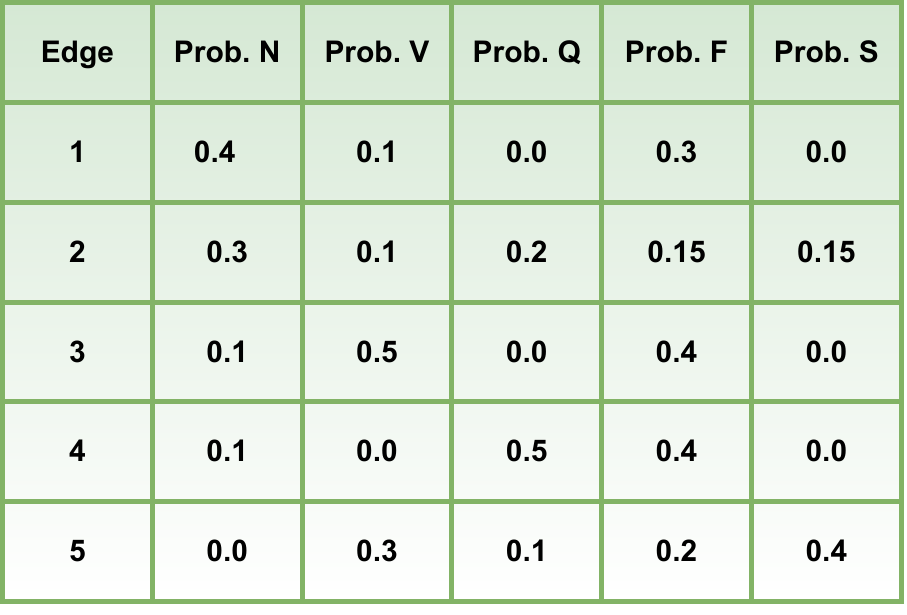}  
  \caption{Input probability distribution}
  \label{fig:case_study_2_PD}
\end{subfigure}
\begin{subfigure}{\linewidth}
  \centering
  \includegraphics[width=\linewidth]{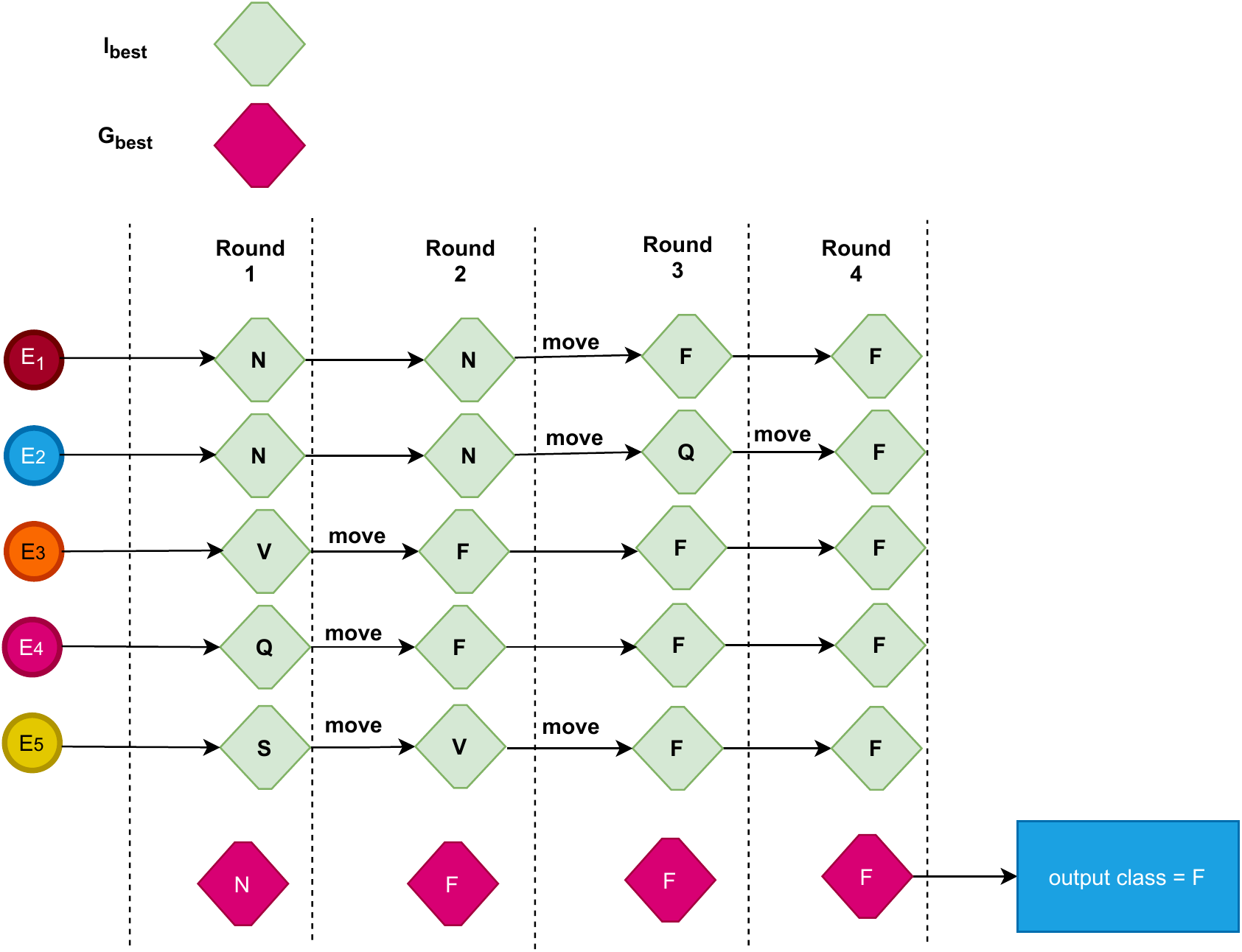}  
  \caption{Proof of Swarm} 
  \label{fig:case_study_2_algo}
\end{subfigure}
\caption{Case Study 2.}
\end{figure}
\fi

\section{Experimental Results}
\label{sec:experiments}

In this section, we show some experimental results of a performance analysis of the proposed PoSw method. To evaluate the proposed method, we trained a convolutional neural network-based five-class classifier for ECG classification in a federated setting. We used five edge deceives in the FL setting to collaboratively train a global model. After training the global model, each edge device downloads the global model and fine tunes it for further classification. We tested each locally tuned global model using a test dataset. Then we used the proposed PoSw method to achieve a consensus among the edge devices for the same test dataset. We used the widely known MIT-BIT arrhythmia dataset~\cite{moody2001} to test the proposed algorithm. The training samples were equally independent and identically distributed among each client. We also kept 1,000 samples for testing which were not used by any client. The PoSw method was implemented with TensorFlow 2.9.0 as the machine learning library, our simulations were run on a computer with an Intel core i-6700HQ CPU and 32 GB RAM. Figure~\ref{fig:round} presents the number of rounds taken by 1,000 simulations of the PoSw method to reach a mutual consensus for each input sample. For most simulations, the PoSw algorithm was not needed because all five edge classifiers predicted the same class label. For other cases, the PoSw algorithm was run to obtain the results, mostly within just one or two rounds and in one case after 20 rounds (which is the maximum number of rounds according to Theorem~\ref{theorem:convergence}).

Figure~\ref{fig:time} presents the time taken (in seconds) by 1,000 software-based simulations of the proposed PoSw method to finally achieve a mutual consensus for each input sample. It can be seen that the proposed PoSw method took on average less than a seconds to achieve a mutual consensus among the participating edge devices.

\begin{figure}[!ht]
\begin{tikzpicture}
  \begin{axis}[
      width=\linewidth,
      height=\iffulledition 6cm \fi,
      grid=both,
      xmin=-50,
      xmax=1050,
      xtick={1, 250, 500, 750, 1000},
      label style={font=\footnotesize},
      xlabel=Simulation ID,
      ylabel=Time per sample (seconds)
    ]
    \addplot 
    table[x=sample,y=time,col sep=comma] {figures/Data/time.csv}; 
  \end{axis}
\end{tikzpicture}
\caption{Performance of the proposed PoSw consensus method (in terms of the time taken by 1,000 software-based simulations to achieve a mutual consensus).}
\label{fig:time}
\end{figure}
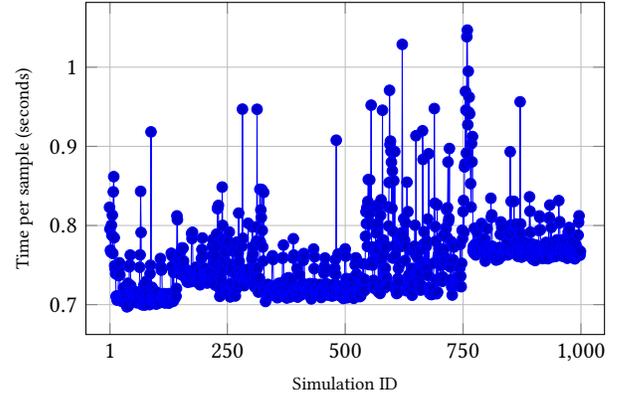

\begin{figure}[!ht]
\begin{tikzpicture}
  \begin{axis}[
      width=\linewidth,
      height=\iffulledition 6cm \fi,
      grid=both,
      xmin=-50,
      xmax=1050,
      xtick={1, 250, 500, 750, 1000},
      label style={font=\footnotesize},
      xlabel=Simulation ID,
      ylabel=Number of Rounds
    ]
    \addplot 
    table[x=sample,y=rounds,col sep=comma] {figures/Data/round.csv};
  \end{axis}
\end{tikzpicture}
\caption{Performance of the proposed PoSw consensus method (in terms of the number of rounds needed by each of 1,000 simulations to achieve a mutual consensus).}
\label{fig:round}
\end{figure}
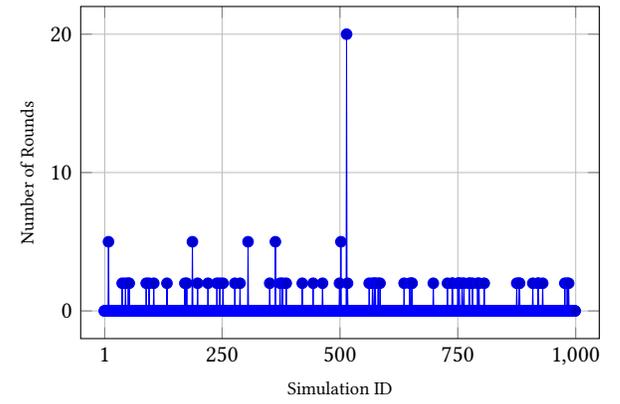

In order to compare the classification performance (in terms of accuracy, defined by $\#(\text{classification errors})/\#(\text{samples})$) of the proposed PoSw-based consensus (ensemble learning) model against the five local models and the FL-based global model, we calculated the accuracy metrics of all the seven models using the same test dataset. Figure~\ref{fig:accuracy} presents the results. It can be observed that the proposed PoSw-based model has the best accuracy, among all models. This indicates that using multiple local models to collectively make a decision can help reduce error rates, even outperforming the FL-based global model.

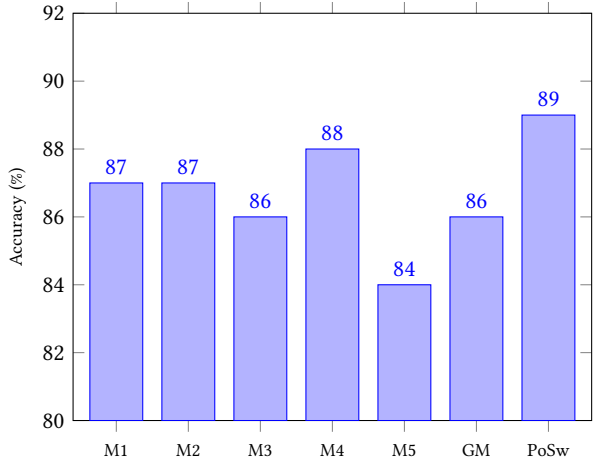
\begin{figure}
\centering
\begin{tikzpicture}
\begin{axis}
[
ybar,
width=\linewidth,
height=\iffulledition 7cm \fi,
ymin=80,
ymax=92,
label style={font=\footnotesize},
ylabel={Accuracy (\%)},
bar width=20pt,
symbolic x coords={M1, M2, M3, M4,M5, GM, PoSw},
xticklabel style={font=\footnotesize, text height=1.5ex}, 
xtick=data,
nodes near coords,
nodes near coords align={vertical}]
\addplot coordinates {(M1,87) (M2,87) (M3,86) (M4,88) (M5,84) (GM,86) (PoSw,89)};
\end{axis}
\end{tikzpicture}
\caption{Comparison of the classification accuracy of the PoSw-based consensus (ensemble learning) model with the global model and the five local models.}
\label{fig:accuracy}
\end{figure}

\section{Comparison}

In this section, we compare key features of our proposed PoSw with the most commonly used ensemble learning methods for classification, i.e., bootstrap aggregation or bagging~\cite{breiman1996bagging}, boosting~\cite{gonzalez2020practical}, stacking~\cite{Pavlyshenko2018} and BFT~\cite{castro2002}. Bagging trains a number of models on different samples of the same training dataset. The predictions made by all ensemble member are then combined using a statistical method like (weighted) majority voting. Bagging can partially solve the tie problem by reducing the probability of having a tie if weighted voting is used, and a rule can be set to decide which result to output in case of a tie. Boosting combines several weak models sequentially by assigning weights to outputs of each model. Then it inputs the incorrect result from the first model in sequence to the subsequent model. Similarly, stacking involves training several weak models and then training a meta model using the outputs of the weak learners. In a simple BFT, \iffulledition as mentioned in Section~\ref{subsec:BFT}, \fi a majority voting is used to determine the final output. In addition to ensemble learning methods, many swarm intelligence (SI) \iffulledition methods \fi can be used to achieve a mutual consensus among multiple parties, but we are not aware of any SI-based \iffulledition methods \fi that can address the tie problem in our application area\footnote{There are other SI-based and blockchain-based methods that can address ties, however, their application is limited in our application area. This is because such methods are not directly applicable to federated learning, where typical features such as a noisy swarm and rankings of clients/edges are lacking in case of SI-based methods and stacking~\cite{hamann2013}, and high computational power in case of blockchain-based methods~\cite{federated2022fedbc}.}.

Table~\ref{tab:comparision} presents the comparison of our proposed method PoSw with the above-mentioned state-of-the-art ensemble learning methods. It can be observed that PoSw has more desirable features by providing a mutual consensus among all the participants unlike bagging and BFT where the result is achieved with a simple (weighted) majority voting. Moreover, in case of FL, boosting is not suitable because of its sequential training nature. Similarly, stacking involves training a meta model using weak learners. In FL, a global model is achieved by aggregation local models, which are then fine-tuned locally. Hence, applying stacking again would make no difference.

\begin{table}[!ht]
\centering
\caption{Comparison of PoSw with selected state-of-the-art ensemble learning and distributed consensus methods}
\begin{tabular}{M{0.3\linewidth}ccc}
\toprule
Method & Mutual Consensus & Tie resolution\\
\midrule
Bagging & No & Partially\\
Boosting & Not Applicable & Not Applicable\\
Stacking & Not Applicable & Not Applicable\\
BFT & No & No\\
Other SI-based methods & Yes & No\\
Other Blockchain-based methods & Yes & No\\
\midrule
PoSw & Yes & Yes\\
\bottomrule
\end{tabular}
\label{tab:comparision}
\end{table}

\iffulledition
\section{Further Discussions}
\label{sec:discussions}

The distributed nature of the consensus provides promising results against various security attacks. For example, in model poisoning attacks, an attacker compromises the local models to alter the performance of the global model. In such cases, since the proposed consensus algorithm effectively considers the outputs of all models, compromising a single model cannot normally alter the final consensus easily. In order to launch a successful attack, the attacker needs to compromise a majority of the edge models, which is expensive and complex. Such collusion attacks are harder to execute in practical applications. Despite the practical difficulties of running collusion attacks, colluding edge clients may have more advanced methods to collaborate and broadcast fake prediction results (class labels predicted) and/or confidence scores to mislead the consensus. This deserves some further investigation.

In addition to collusion attacks, sharing prediction results with confidence scores among the peers could lead to leakage of more information about local training data, therefore a higher level of privacy concerns. This problem could be addressed using mechanisms such as differential privacy~\cite{ji2014differential} and homomorphic encryption~\cite{naehrig2011can}. However, such mechanisms come with trad-offs between run-time performance, privacy and utility. Hence, more studies are needed to investigate how much additional information can be inferred from the outputs from local models and what can be done to mitigate such new privacy concerns.
\fi

\section{Conclusions}
\label{conclusions}

In this article, we proposed a novel distributed consensus algorithm called PoSw (Proof of Swarm) to achieve ensemble learning in federated learning applications. Using the proposed PoSw method, distributed peers can always converge to reach a consensus in $K(K-1)$ steps, where $K$ is the number of classes of the classification problem. Additionally, the proposed PoSw method can efficiently solve tie events. Unlike the classical distributed consensus algorithm, such as Byzantine fault tolerance the proposed algorithm does not makes consensus based on a simple majority voting, instead, it considers confidence scores of predicted class labels of all peer classifiers and tries to achieve a more optimised consensus decision among all the peers in the network. \iffulledition We provide two case studies to show the capability of proposed algorithm to achieve efficient and sub-optimum consensus among peers. \fi Using experimental results of an ECG classification task with five classes, we show that the proposed PoSw-based ensemble learning model outperformed all local models and also the FL-based global model, in terms of the overall accuracy. \iffulledition We also discuss some data security and privacy related issues of the proposed method, which help define future work.\fi

\begin{acks}
This study was supported by the grant (REF LABEX/EQUIPEX), a French State fund managed by the National Research Agency under the frame program ``Investissements d'Avenir'' I-SITE ULNE / ANR-16-IDEX-0004 ULNE.
\end{acks}

\bibliographystyle{ACM-Reference-Format}
\bibliography{main}

\end{document}